\documentclass[10pt,twocolumn,letterpaper]{article}

\usepackage{cvpr}
\usepackage{times}
\usepackage{epsfig}
\usepackage{graphicx}
\usepackage{amsmath}
\usepackage{amssymb}
\usepackage{xcolor}
\usepackage{microtype}
\usepackage{balance}

\usepackage[pagebackref=true,breaklinks=true,letterpaper=true,colorlinks,bookmarks=false]{hyperref}

\usepackage{color,graphicx,amssymb,amsmath}
\usepackage{amsthm}
\usepackage{etoolbox}
\usepackage{bm}
\usepackage{mathtools}
\usepackage{algorithm}
\usepackage{algorithmic}
\usepackage{subfigure}
\usepackage{xspace}
\usepackage{multirow}
\usepackage[capitalise]{cleveref}
\usepackage{array}
\usepackage{nicefrac}
\usepackage{url}

\newcommand{\tabincell}[2]{\begin{tabular}{@{}#1@{}}#2\end{tabular}}
\usepackage{threeparttable}

\newcommand{\MyMapTemplatePrefixc}[4]{\expandafter#1\csname#3#4\endcsname{#2{#4}}} 
\forcsvlist{\MyMapTemplatePrefixc {\def} {\mathcal}{mc}} {A,B,C,D,E,F,G,H,I,J,K,L,M,N,O,P,Q,R,S,T,U,V,W,X,Y,Z}  

\def\hu{\tilde{u}}

\def\ga{\alpha}
\def\gb{\beta}

\def\gep{\epsilon}
\def\ggm{\gamma}
\def\gl{\lambda}

\def\gt{\theta}

\def\gD{\Delta}

\def\gO{\Omega}

\def\hga{\hat{\alpha}}
\def\hgb{\hat{\beta}}
\def\hgl{\hat{\lambda}}

\def\bbN{{\mathbb N}}
\def\bbR{{\mathbb R}}

\def\st{\mbox{subject to}}

\newcommand{\grad}{\nabla}
\newcommand{\trans}{^{T}}
\newcommand{\inprod}[1]{\langle #1 \rangle}

\newcommand{\opt}{^{*}}
\newcommand{\itk}{_{k}}
\newcommand{\kp}{_{k+1}}
\newcommand{\km}{_{k-1}}
\newcommand{\reso}[1]{\gD #1 ^* }
\newcommand{\resp}[1]{\gD #1 ^+ }
\usepackage{nicefrac}
\newcommand{\half}{\nicefrac{1}{2}}

\newtheorem{proposition}{Proposition}
\crefname{proposition}{Proposition}{Proposition}
\newtheorem{theorem}{Theorem}
\crefname{theorem}{Theorem}{Theorem}
\newtheorem{lemma}{Lemma}
\crefname{lemma}{Lemma}{Lemma}
\crefname{myrl}{Rule}{Rule}
\newtheorem{assumption}{Assumption}
\crefname{assumption}{Assumption}{Assumption}

\newcommand{\tom}[1]{}

\newcommand*{\affaddr}[1]{#1} 
\newcommand*{\affmark}[1][*]{\textsuperscript{#1}}

 \cvprfinalcopy 

\def\cvprPaperID{3685} 

\ifcvprfinal\fi
\begin{document}

\newcommand{\cs}[1]{}

\title{Adaptive Relaxed ADMM:  Convergence Theory and Practical Implementation}

\author{
Zheng Xu\affmark[1]\thanks{{\tt\small xuzh@cs.umd.edu}}, \,\! M\'{a}rio A. T. Figueiredo\affmark[2],  Xiaoming Yuan\affmark[3],  Christoph Studer\affmark[4], and Tom Goldstein\affmark[1]\\
\affaddr{\affmark[1]Department of Computer Science, University of Maryland, College Park, MD}\\
\affaddr{\affmark[2]Instituto de Telecomunica\c{c}\~{o}es, Instituto Superior T\'{e}cnico, Universidade de Lisboa, Portugal}\\
\affaddr{\affmark[3]Department of Mathematics, Hong Kong Baptist University, Kowloon Tong, Hong Kong}\\
\affaddr{\affmark[4]School of Electrical and Computer Engineering, Cornell University, Ithaca, NY}\\
}

\maketitle

\begin{abstract}
Many modern computer vision and machine learning applications rely on solving difficult optimization problems that involve non-differentiable objective functions and constraints. The alternating direction method of multipliers (ADMM) is a widely used approach to solve such problems. Relaxed ADMM is a generalization of ADMM that often achieves better performance, but its efficiency depends strongly on algorithm parameters that must be chosen by an expert user. We propose an adaptive method that automatically tunes the key algorithm parameters to achieve optimal performance without user oversight. Inspired by recent work on adaptivity, the proposed \emph{adaptive relaxed ADMM} (ARADMM) is derived by assuming a Barzilai-Borwein style linear gradient. A detailed convergence analysis of ARADMM is provided, and numerical results on several applications demonstrate fast practical convergence.
\end{abstract}


\section{Introduction}
Modern methods in computer vision and machine learning often require solving\cs{above you use complex; maybe stick to one term?} difficult optimization problems involving non-differentiable objective functions and constraints.
Some popular applications include sparse models \cite{wright2009robustface,yang2009linear,elhamifar2009sparse,mairal2014sparse}, low-rank models \cite{wright2009robust,harchaoui2012large,xu2015bmvc,li2017domain}, and support vector machines (SVMs) \cite{cortes1995support,chang2011libsvm}. The alternating direction method of multiplier (ADMM) is one of the most prominent optimization tools to solve such problems, and tackles problems in the following form:
\begin{eqnarray}
\min_{u\in \bbR^n,v\in \bbR^m}  h(u) + g(v),~~~~\st~~  Au+Bv = b. \label{eq:prob}
\end{eqnarray}
Here, $h:\bbR^n\rightarrow \bbR$ and $g:\bbR^m\rightarrow \bbR$ are closed, proper, and convex functions, $A\in \bbR^{p \times n}$, \mbox{$B\in \bbR^{p \times m}$}, and \mbox{$b \in \bbR^p$}.
ADMM was first introduced in \cite{glowinski1975approximation} and \cite{gabay1976dual}, and has found applications in a variety of optimization problems in machine learning, image processing, computer vision, wireless communications, and many other areas \cite{boyd2011admm,goldstein2014fast}.

Relaxed ADMM is a popular  practical variant of ADMM,
and proceeds with the following steps:
\begin{align}
u\kp & = \arg\min_{u} h(u) +  \frac{\tau\itk}{2} \left\| b-Au-Bv\itk + \frac{\gl\itk}{\tau\itk} \right\|^2\! \label{eq:updateu}\\
\hu\kp & = \ggm\itk A u\kp + (1-\ggm\itk) (b-Bv\itk)\label{eq:relax}\\
v\kp &= \arg\min_{v} g(v) + \frac{\tau\itk}{2} \left\| b-\hu\kp-Bv + \frac{\gl\itk}{\tau\itk}\right\|^2\! \label{eq:updatev}\\
\gl\kp &= \gl\itk +\tau\itk (b - \hu\kp -B v\kp). \label{eq:updatedual}
\end{align}
Here, $\gl_k \! \in \! \bbR^p$ denotes the dual variables (Lagrange multipliers) on iteration $k$, and $(\tau\itk, \ggm\itk)$ are sequences of penalty and relaxation parameters. Relaxed ADMM coincides with the original non-relaxed version if $\ggm\itk=1$.



Convergence of (relaxed) ADMM is guaranteed under fairly general assumptions  \cite{eckstein1992douglas,he2012con,he2015non,fang2015generalized}, if the penalty and relaxation parameters are held constant.
However, the practical performance of ADMM depends strongly on the choice of these parameters, as well as on the problem being solved. Good penalty choices are known for certain ADMM formulations, such as strictly convex quadratic problems~\cite{raghunathan2014alternating,ghadimi2015optimal}, and for the gradient descent parameter in the ``linearized'' ADMM \cite{lin2011linearized,liu2013linearized}.

Adaptive penalty methods (in which the penalty parameters are tuned automatically as the algorithm proceeds) achieve good performance without user oversight.  For non-relaxed ADMM, the authors of \cite{he2000alternating} propose methods that modulate the penalty parameter so that the primal and dual residuals (i.e., derivatives of the Lagrangian with respect to primal and dual variables) are of approximately equal size. This ``residual balancing'' approach has been generalized to work with preconditioned variants of ADMM \cite{goldstein2015adaptive} and distributed ADMM \cite{song2015fast}.
 In \cite{xu2016adaptive}, a spectral penalty parameter method is proposed that uses the local curvature of the objective to achieve fast convergence.  All of these methods are specific to (non-relaxed) {\it vanilla} ADMM, and do not apply to the more general case involving a relaxation parameter.


\subsection{Overview \& contributions}

In this paper, we study adaptive parameter choices for the relaxed ADMM that jointly and automatically tune both the penalty parameter $\tau\itk$ and relaxation parameter $\ggm\itk$.  
In Section \ref{sec:converge}, we address theoretical questions about the convergence of ADMM with non-constant penalty and relaxation parameters.  In Section \ref{sec:choice}, we discuss practical methods for choosing these parameters. In Section \ref{sec:app}, we apply the proposed ARADMM to several problems in machine learning, computer vision, and image processing.  Finally, in Section~\ref{sec:exp}, we compare ARADMM to other ADMM variants and examine the benefits of the proposed approach for real-world regression, classification, and image processing problems. 
%

\section{Related work}
Sparse and low rank methods are widely used in computer vision \cite{wright2009robustface,yang2009linear,elhamifar2009sparse,wright2009robust,harchaoui2012large,mairal2014sparse,xu2015bmvc,li2017domain}, machine learning \cite{efron2004least,zou2005regularization,schmidt2007fast,fan2008liblinear,liu2009large}, and image processing \cite{rudin1992nonlinear,goldstein2014fast}. ADMM has been extensively applied to solve such problems \cite{boyd2011admm,goldstein2014fast,xu2016adaptive,xu2016empirical},
and has recently found applications in neural networks \cite{zhang2016efficient,taylor2016training}, tensor decomposition~\cite{goldfarb2014robust,lu2016tensor,xu2016non}, structure from motion \cite{goldstein2016shapefit}, and other vision problems.

The $O(1/k)$ convergence rate of non-relaxed ADMM is established under mild conditions for convex problems~\cite{he2012con,he2015non}. The  $O(1/k^2)$ convergence rate is discussed in \cite{goldfarb2013fast,goldstein2014fast,kadkhodaie2015accelerated,tian2016faster}, where at least one of the functions is assumed either strongly convex or smooth. 
For the general relaxed ADMM formulation, a $O(1/k)$  convergence rate is provided under mild conditions \cite{fang2015generalized}. Linear convergence can be achieved with strong convexity assumptions~\cite{davis2014faster,nishihara2015general,giselsson2016linear}.  
All of these results assume constant parameters---it is considerably harder to prove convergence when the algorithm parameters are adaptive.

Fixed optimal parameters are discussed in the literature. For the specific case in which the objective is quadratic, a criterion is proposed in \cite{raghunathan2014alternating,ghadimi2015optimal}. The authors of \cite{nishihara2015general}  suggest a grid search and semidefinite programming based method to determine the optimal relaxation and penalty parameters.  These methods, however, make strong assumptions about the objective and require knowledge of condition numbers. 

Adaptive penalty methods are proposed to accelerate the practical convergence of non-relaxed ADMM \cite{he2000alternating,xu2016adaptive}. For the relaxation parameter,  it has been suggested in \cite{eckstein1992douglas} that over-relaxation ($\ggm\in(1,2)$) may accelerate convergence and $\ggm=1.5$ achieves faster convergence in a specific distributed computing application. The proposed ARADMM simultaneously adapts both the penalty and the relaxation parameter, thus being fully automated. 


\section{Convergence theory}
\label{sec:converge}
We study conditions under which ADMM converges with adaptive penalty and relaxation parameters.   Our approach utilizes the variational inequality (VI) methods put forward in  \cite{he2000alternating,he2012con,he2015non}.  
Our results measure convergence using the primal and dual ``residuals," which are defined as
\begin{align}\label{eq:res}
r\itk = b - A u\itk - B v\itk \ \text{and}\ d\itk = \tau\itk A^TB(v\itk-v\km).
\end{align}
It has been observed that these residuals approach zero as the algorithm approaches a true solution \cite{boyd2011admm}.
Typically, the iterative process is stopped if
\begin{equation}\label{eq:stop}
\begin{split}
& \|r_k\|  \leq \gep^{tol} \max\{\|Au_k\|, \|Bv_k\|, \|b\| \}  \\
& \text{and} ~~\|d_k\| \leq \gep^{tol} \| A^{T} \gl_k\|,
\end{split}
\end{equation}
where $\gep^{tol} > 0$ is the stopping tolerance \cite{boyd2011admm}.  
For this reason, it is important to know that the method converges in the sense that the residuals approach zero as $k\to\infty.$
%

In the sequel, we prove that relaxed ADMM converges in the residual sense, provided that the algorithm parameters satisfy one of the following two assumptions.
\begin{assumption}\label{as1}
The relaxation sequence $\ggm\itk$ and penalty sequence $\tau\itk$ satisfy
\begin{equation}
\begin{split}
& 1 \leq \ggm\itk < 2, \, \lim_{k\rightarrow \infty} 1/\tau\itk^2 < \infty, \,
\sum_{k=1}^{\infty} \eta\itk^2 < \infty,
\\
& \quad \text{where} \quad \eta\itk^2 =
\frac{\ggm\itk}{(2-\ggm\itk)} \max\left(\tau\itk^2/\tau\km^2, \, 1\right)-1.
\end{split}
\end{equation}
\end{assumption}

\begin{assumption}\label{as2}
The relaxation sequence $\ggm\itk$ and penalty sequence $\tau\itk$ satisfy
\begin{equation}
\begin{split}
& 1 \leq \ggm\itk < 2, \, \lim_{k\rightarrow \infty}\tau\itk^2 < \infty, \,
\sum_{k=1}^{\infty} \gt\itk^2 < \infty,
 \\
& \quad\text{where}\quad \gt\itk^2 =
\frac{\ggm\itk}{(2-\ggm\itk)} \max\left(\tau\km^2/\tau\itk^2, \, 1\right)-1.
\end{split}
\end{equation}
\end{assumption}

In Section \ref{sec:proofs}, we prove adaptive relaxed ADMM converges if the algorithm parameters satisfy either Assumption \ref{as1} or Assumption \ref{as2}.  Before presenting the proof, we show how to choose the relaxation parameters that lead to efficient performance in practice.  




\section{ARADMM: Adaptive relaxed ADMM}\label{sec:choice}
Spectral stepsize selection methods for vanilla ADMM were discussed in \cite{xu2016adaptive}.  Here, we modify the adaptive ADMM framework in two important ways.  First, we discuss the selection of penalty parameters in the presence of the relaxation term.  Second, we discuss adaptive methods also for automatically selecting the relaxation parameter.

The proposed method works by assuming a local linear model for the dual optimization problem, and then selecting an optimal stepsize under this assumption.  A safeguarding method is adopted to ensure that bad stepsizes are not chosen in case these linearity assumptions fail to hold.

\subsection{Dual interpretation of relaxed ADMM}
\label{sec:ADMM_DRS}
We derive our adaptive stepsize rules by examining the close relationship between relaxed ADMM and the relaxed Douglas-Rachford Splitting (DRS) \cite{eckstein1992douglas,davis2014faster,giselsson2016linear}. The dual of the general constrained problem~\eqref{eq:prob} is
\begin{equation}
\min_{\zeta\in \bbR^p}  \underbrace{h^*(A^{T}\zeta) - \inprod{\zeta, b}}_{ \hat{h}(\zeta)} + \underbrace{g^*(B^{T}\zeta)}_{\hat{g}(\zeta)},
\label{eq:def_Hhat_Ghat}
\end{equation}
with $f^*$ denoting the Fenchel conjugate of $f$, defined as  $f^*(y) = \sup_{x} \langle x, y\rangle - f(x)$ \cite{Rockafellar}.

The relaxed DRS algorithm solves \eqref{eq:def_Hhat_Ghat} by generating two sequences, $(\zeta\itk)_{k\in\bbN} $ and $(\hat \zeta\itk)_{k\in\bbN},$ according to
{\small
\begin{align}
0  \in & \frac{\hat{\zeta}\kp-\zeta\itk}{\tau\itk}  + \partial \hat h(\hat{\zeta}\kp) +  \partial \hat g(\zeta\itk), \label{eq:dr1}\\
0  \in & \frac{{\zeta}\kp-\zeta\itk}{\tau\itk} + \ggm\itk\, \partial \hat h(\hat{\zeta }\kp)
\nonumber \\ & \quad
- (1-\ggm\itk)\partial \hat g(\zeta\itk) + \partial \hat g(\zeta\kp), \label{eq:dr2}
\end{align}
}%
where $\ggm\itk$ is a relaxation parameter, and $\partial f(x)$ denotes the subdifferential of $f$ evaluated at $x$ \cite{Rockafellar}.
Referring back to ADMM in \eqref{eq:updateu}--\eqref{eq:updatedual}, and defining  $\hat{\gl}\kp = \gl\itk  +\tau\itk (b - A u\kp -B v\itk)$, the sequences $(\lambda\itk)_{k\in\bbN} $ and $(\hat \lambda\itk)_{k\in\bbN}$ satisfy  the same conditions  \eqref{eq:dr1} and \eqref{eq:dr2} as $(\zeta\itk)_{k\in\bbN} $ and $(\hat \zeta\itk)_{k\in\bbN}$, thus ADMM  for the problem~\eqref{eq:prob} is equivalent to DRS on its dual \eqref{eq:def_Hhat_Ghat}. A detailed proof of this is provided in the supplementary material.

\subsection{Spectral adaptive stepsize rule}
\label{sec:dr2curv}
Adaptive stepsize rules of the ``spectral'' type were originally proposed for simple gradient descent on smooth problems by Barzilai and Borwein \cite{barzilai1988two}, and have been found to dramatically outperform constant stepsizes in many applications \cite{fletcher2005barzilai,wright2009sparse}.  Spectral stepsize methods work by modeling the gradient of the objective as a linear function, and then selecting the optimal stepsize for this simplified linear model.

Spectral methods were recently used to determine the penalty parameter for the non-relaxed ADMM in \cite{xu2016adaptive}.
Inspired by that work, we derive spectral stepsize rules assuming a linear model/approximation for $\partial\hat h(\hat \zeta)$ and  $\partial \hat g(\zeta)$ at iteration $k$ given by
\begin{equation}
\partial \hat h(\hat \zeta) = \alpha\itk \, \hat \zeta + \Psi\itk ~~~~~\text{and}~~~~~\partial \hat g(\zeta) = \beta\itk \, \zeta + \Phi\itk, \label{eq:linearb2}
\end{equation}
where  $\ga\itk > 0$, $\gb\itk >0$ are local curvature estimates of $\hat{h}$ and $\hat{g}$, respectively, and $\Psi\itk, \Phi\itk \subset \bbR^p$.  Once we obtain these curvature estimates, we will exploit the following simple proposition whose proof is given in the supplementary material.

\begin{proposition} \label{prop}
Suppose the DRS steps \eqref{eq:dr1}--\eqref{eq:dr2} are applied to  problem \eqref{eq:def_Hhat_Ghat}, where (omitting iteration $k$ from $\alpha\itk, \beta\itk, \Psi\itk, \Phi\itk$ to lighten the notation in what follows)
\begin{equation}
\partial \hat h(\hat \zeta) = \alpha \, \hat \zeta + \Psi ~~~~~\text{and}~~~~~\partial \hat g(\zeta) = \beta \, \zeta + \Phi . \label{eq:linearb}
\end{equation}
Then, the residual of $\, \hat{h}(\zeta\kp) + \hat{g}(\zeta\kp)$ will be zero if $\tau$ and $\ggm$ are chosen to satisfiy
 $\ggm\itk = 1 + \frac{1+\ga\gb\tau_k^2}{(\ga+\gb)\tau_k}.$
\label{rl:spectral}
\end{proposition}

Our adaptive method works by fitting a linear model to the
gradient (or subgradient) of our objective, and then using Proposition \ref{prop} to select an optimal stepsize pair that obtains zero residual on the model problem.   For our convergence theory to hold, we need $\ggm<2.$  For fixed values of $\alpha$ and $\beta,$ the minimal value of $\ggm\itk$ that is still optimal for the linear model occurs if we choose
 \begin{align}
\tau\itk = \arg\min_{\tau} \frac{1+\ga\gb\tau^2}{(\ga+\gb)\tau}
 =  1/\sqrt{\ga\gb}. \label{eq:opttau}
\end{align}
Note this is the same ``optimal'' penalty parameter proposed for non-relaxed ADMM in \cite{xu2016adaptive}.  Under this choice of $\tau_k,$ we then have the ``optimal'' relaxation parameter
\begin{align}
\ggm\itk = 1 + \frac{1+\ga\gb\tau^2}{(\ga+\gb)\tau} = 1 + \frac{2\sqrt{\ga\gb}}{\ga + \gb} \le 2. \label{eq:optggm}
\end{align}


\subsection{Estimation of stepsizes}
\label{sec:curvestim}
We now propose a simple method for fitting a linear model to the dual objective terms so that the formulas in Section \ref{sec:dr2curv} can be used to obtain stepsizes.  Once these linear models are formed, the optimal penalty parameter and relaxation term can be calculated by \eqref{eq:opttau} and \eqref{eq:optggm}, thanks to the equivalence of relaxed ADMM and DRS.

In what follows, we let  $\hat \ga\itk = 1/\alpha\itk$ and $\hat \gb\itk = 1/\beta\itk$  to simplify notation. The optimal stepsize choice is then written as $\tau\itk = (\hat\alpha\itk\, \hat \beta\itk)^{1/2}$ and $\ggm\itk = 1 + \frac{2\sqrt{\hat\ga\itk\hat\gb\itk}}{\hat\ga\itk + \hat\gb\itk}$.

 The estimation of $\hat \ga\itk$ and $\hat \gb\itk$ for the dual components $\hat{h}(\hat\lambda\itk)$ and $\hat{g}(\lambda\itk)$ at the $k$-th iteration of primal ADMM has been described in \cite{xu2016adaptive}.  It is easy to verify that the model parameters $\hat \ga\itk$ and $\hat \gb\itk$ of relaxed ADMM can be estimated based on the results from iteration $k$ and an older iteration $k_0<k$ in a similar way.
If we define
\begin{align} \label{eq:du}
\gD \hat{\gl}\itk &:= \hat{\gl}\itk - \hat{\gl}_{k_0}  \quad \text{and} \quad \gD \hat h\itk :
= A(u\itk-u_{k_0}),
\end{align}
then the parameter $\hat \ga\itk$ is obtained from the formula
\begin{align}
& \hat{\ga}\itk =
\begin{cases}
\hat{\ga}\itk^{\mbox{\scriptsize MG}}&~~\text{if}~~2 \,\hat{\ga}\itk^{\mbox{\scriptsize MG}} > \hat{\ga}\itk^{\mbox{\scriptsize SD}} \\
\hat{\ga}\itk^{\mbox{\scriptsize SD}} - \hat{\ga}\itk^{\mbox{\scriptsize MG}} /2  &~~\text{otherwise,}
\end{cases}\label{eq:alpha} \\
& \hga\itk^{\mbox{\scriptsize SD}} = \frac{\inprod{\gD \hat{\gl}\itk, \gD \hat{\gl}\itk}}{\inprod{\gD \hat h\itk, \gD \hat{\gl}\itk}}
\,\, \text{ and }\,\,
\hga\itk^{\mbox{\scriptsize MG}} = \frac{\inprod{\gD \hat h\itk, \gD \hat{\gl}\itk}}{\inprod{\gD \hat h\itk, \gD \hat h\itk}}.
\end{align}
For a detailed derivation of these formulas, see \cite{xu2016adaptive}.

The spectral stepsize $\hat{\gb}\itk$ of $\hat{g}(\gl\itk)$ is similarly estimated with $ \gD \hat g\itk \! := \! B(v\itk-v_{k_0})$, and $\gD \gl\itk \! := \! \gl\itk -\gl_{k_0}$.
It is important to note that  $\hat{\ga}\itk $ and $\hat{\gb}\itk$ are obtained from the iterates of ADMM alone, i.e., our scheme does not require the user to supply the dual problem.

\subsection{Safeguarding} \tom{This is WORDY, kill for lenght}
Spectral stepsize methods for simple gradient descent are paired with a backtracking line search to guarantee convergence in case the linear model assumptions break down and an unstable stepsize is produced.  ADMM methods have no analog of backtracking.  Rather, we adopt the correlation criterion proposed in \cite{xu2016adaptive} to test the validity of the local linear assumption, and only rely on the adaptive model when the assumptions are deemed valid.  To this end, we define
\begin{align}
\ga^{\mbox{\scriptsize cor}}\itk = \frac{\inprod{\gD \hat h\itk, \gD \hat{\gl}\itk}}{ \| \gD \hat h\itk\| \, \| \gD \hat{\gl}\itk\| } \ \, \text{and} \, \
\gb^{\mbox{\scriptsize cor}}\itk = \frac{\inprod{\gD \hat g\itk, \gD \gl\itk}}{ \| \gD \hat g\itk\| \, \| \gD \gl\itk\| }. \label{eq:corr}
\end{align}
When the model assumptions \eqref{eq:linearb} hold perfectly, the vectors $\gD \hat h\itk$ and $\gD \hat{\gl}\itk$ should be highly correlated and we get $\ga^{\mbox{\scriptsize cor}}\itk = 1.$ When $\ga^{\mbox{\scriptsize cor}}\itk$ or $\gb^{\mbox{\scriptsize cor}}\itk$ is small, the model assumptions are invalid and the spectral stepsize may not be effective.

The proposed method uses the following update rules
\begin{equation} \small
\hspace{-2mm}
\tau\kp =
\begin{cases}
\sqrt{\hat{\ga}\itk \hat{\gb}\itk} &~~\text{if}~~ \ga^{\mbox{\scriptsize cor}}\itk > \gep^{\mbox{\scriptsize cor}}~~\text{and}~~\gb^{\mbox{\scriptsize cor}}\itk > \gep^{\mbox{\scriptsize cor}}\\
\hat{\ga}\itk &~~\text{if}~~ \ga^{\mbox{\scriptsize cor}}\itk > \gep^{\mbox{\scriptsize cor}}~~\text{and}~~\gb^{\mbox{\scriptsize cor}}\itk \leq \gep^{\mbox{\scriptsize cor}}\\
\hat{\gb}\itk &~~\text{if}~~ \ga^{\mbox{\scriptsize cor}}\itk \leq \gep^{\mbox{\scriptsize cor}}~~\text{and}~~\gb^{\mbox{\scriptsize cor}}\itk > \gep^{\mbox{\scriptsize cor}}\\
\tau\itk &~~\text{otherwise},
\end{cases} \label{eq:final}
\end{equation}
\begin{equation} \small
\hspace{-1mm}
\ggm\kp =
\begin{cases}
1 + \frac{2\sqrt{\hat\ga\itk\hat\gb\itk}}{\hat\ga\itk + \hat\gb\itk}
 &~\text{if}~ \ga^{\mbox{\scriptsize cor}}\itk > \gep^{\mbox{\scriptsize cor}}~\text{and}~\gb^{\mbox{\scriptsize cor}}\itk > \gep^{\mbox{\scriptsize cor}}\\
 1.9 &~\text{if}~ \ga^{\mbox{\scriptsize cor}}\itk > \gep^{\mbox{\scriptsize cor}}~~\text{and}~~\gb^{\mbox{\scriptsize cor}}\itk \leq \gep^{\mbox{\scriptsize cor}}\\
 1.1 &~\text{if}~ \ga^{\mbox{\scriptsize cor}}\itk \leq \gep^{\mbox{\scriptsize cor}}~~\text{and}~~\gb^{\mbox{\scriptsize cor}}\itk > \gep^{\mbox{\scriptsize cor}}\\
 1.5 &~\text{otherwise},
\end{cases} \label{eq:final2}
\end{equation}
where $\gep^{\mbox{\scriptsize cor}}$ is a quality threshold for the curvature estimates, while $\hat{\ga}\itk$ and $\hat{\gb}\itk$ are the spectral stepsizes estimated in \cref{sec:curvestim}. The update for $\tau\kp$ only uses model parameters that have been accurately estimated.  When the model is effective for $h$ but not $g,$ we use a large $\ggm\itk=1.9$ to make the~$v$ update conservative relative to the~$u$ update.  When the model is effective for $g$ but not $h,$ we use a small $\ggm\itk=1.1$ to make the $v$ update aggressive relative to the $u$ update.

\subsection{Applying convergence guarantee}


Our convergence theory requires either \cref{as1} or \cref{as2} to be satisfied, which suggests that convergence is guaranteed under ``bounded adaptivity'' for both penalty and relaxation parameters. These conditions can be guaranteed by explicitly adding constraints to the stepsize choice in ARADMM. 

To guarantee convergence, we simply replace the parameter updates \eqref{eq:final} and \eqref{eq:final2} with
\begin{equation}
\begin{split}
\hat\tau\kp = & \min\left\{\tau\kp , \, \left(1+\nicefrac{C_{cg}}{k^2}\right) \tau\itk\right\} \\
\hat\ggm\kp = & \min\left\{\ggm\kp , \, 1+\nicefrac{C_{cg}}{k^2}\right\}\!, \label{eq:cgcon}
\end{split}
\end{equation}
where $C_{cg}$ is some (large) constant.
It is easily verified that the parameter sequence $( \hat\tau\itk, \hat\ggm\itk )$ satisfies \cref{as1}. In practice, the update schemes \eqref{eq:final} and \eqref{eq:final2} converges reliably without explicitly enforcing these conditions. We use a very large $C_{cg}$ such that the conditions are not triggered in the first few thousand iterations and provide these constraints for theoretical interests.

\setlength{\textfloatsep}{7pt}
\begin{algorithm}[tbp]
\caption{Adaptive relaxed ADMM (ARADMM)}
\label{alg}
\begin{algorithmic}[1]
\REQUIRE  initialize $v_0$, $\gl_0$, $\tau_0$, $\ggm_0$, and $k_0 \! = \! 0$
\WHILE{not converge by \eqref{eq:stop} \textbf{and} $k <\text{maxiter}$}
\STATE Perform relaxed ADMM, as in \eqref{eq:updateu}--\eqref{eq:updatedual}
\IF{$\text{mod}(k, T_{f}) = 1$}
\STATE $\hgl\kp = \gl\itk +\tau\itk (b - A u\kp -B v\itk)$
\STATE Compute spectral stepsizes $\hga\itk,\hgb\itk$ using~\eqref{eq:alpha}
\STATE Estimate correlations $\ga\itk^{\mbox{\scriptsize cor}},\gb\itk^{\mbox{\scriptsize cor}}$ using \eqref{eq:corr}
\STATE Update $\tau\kp,\ggm\kp$ using \eqref{eq:final} and \eqref{eq:final2}
\STATE Bound $\tau\kp,\ggm\kp$ using \eqref{eq:cgcon}
\STATE  $k_0 \gets k$
\ELSE
\STATE $\tau\kp \gets \tau\itk$ and $\ggm\kp \gets \ggm\itk$
\ENDIF
\STATE $k \gets k+1$
\ENDWHILE
\end{algorithmic}

\end{algorithm}

\subsection{ARADMM algorithm}
The complete \textit{adaptive relaxed ADMM} (ARADMM) is shown in \cref{alg}. We suggest only updating the stepsize every $T_f = 2$ iterations. We suggest a fixed safeguarding threshold $\gep^{\mbox{\scriptsize cor}}=0.2,$ which is used in all the experiments in Section \ref{sec:app}.  The overhead of the adaptive scheme is modest, requiring only a few inner product calculations.

\section{Proofs of convergence theorems} \label{sec:proofs}
We now prove that relaxed ADMM converges under Assumption \ref{as1} or \ref{as2}.
Let
\begin{align}
y =
\begin{pmatrix}
u \\
v \\
\end{pmatrix} \in \bbR^{n + m}, \
z=
\begin{pmatrix}
u \\
v \\
\gl\\
\end{pmatrix} \in \bbR^{n +m +p}. \label{eq:defyz}
\end{align}
We use $y\itk = (u\itk, v\itk)\trans$ and $z\itk = (u\itk, v\itk, \gl\itk)\trans$
to denote iterates, and $y\opt = (u\opt, v\opt)\trans$ and $z\opt = (u\opt, v\opt, \gl\opt)\trans$ denote optimal solutions. Set $\resp{z}\itk = (\resp{u}\itk, \resp{v}\itk, \resp{\gl}\itk):=z\kp-z\itk$, 
and $\reso{z}\itk = (\reso{u}\itk, \reso{v}\itk, \reso{\gl}\itk):=z\opt-z\itk$,
and define
\begin{align}
f(y) = h(u) + g(v),  \quad
F(z) = \begin{pmatrix}
-A\trans \gl\\
 -B\trans \gl\\
  A u + B v - b\\
  \end{pmatrix}\!. \label{eq:deffF}
\end{align}
Notice that $F(z)$ is monotone, which means $\forall z, z', (z-z')\trans(F(z)-F(z')) \geq 0 $.

Problem formulation \eqref{eq:prob} can be reformulated as a variational inequality (VI).  The optimal solution $ z\opt$ satisfies
\begin{align}
 \forall z, \,\,\, f(y)-f(y\opt) + (z-z\opt)\trans F(z\opt) \geq 0. \label{eq:optvi}
\end{align}
Likewise, the ADMM iterates produced by steps \eqref{eq:updateu} and \eqref{eq:updatev} satisfy the variational inequalities
\begin{align}
\forall u, \,\,\,  & h(u) - h(u\kp) +(u-u\kp)\trans \nonumber\\
& (\tau\itk A \trans (Au\kp + Bv\itk -b) - A\trans \gl\itk) \geq 0,  \label{eq:hvi} \\
\forall v, \,\,\,  & g(v) - g(v\kp) + (v-v\kp)\trans \nonumber\\
&(\tau\itk B \trans (\hu\kp + Bv\kp -b) - B\trans \gl\itk) \geq 0.  \label{eq:gvi}
\end{align}
Using the definitions of $y$, $z$, $f(y)$, and $F(z)$ in (\ref{eq:defyz}, \ref{eq:deffF}), $\gl$ in \eqref{eq:updatedual},
and $\hu$ in \eqref{eq:relax},
VI \eqref{eq:hvi} and \eqref{eq:gvi} combine to yield
{\small
\begin{align}
& f(y)-f(y\kp) + (z-z\kp)\trans \left( F(z\kp) + \gO(\resp{z}\itk, \tau\itk, \ggm\itk)\right) \geq 0, \nonumber \\
& \gO(\resp{z}\itk, \tau\itk, \ggm\itk) = \begin{pmatrix}
\frac{\ggm\itk-1}{\ggm\itk}A\trans \resp{\gl}\itk-\frac{\tau\itk}{\ggm\itk}A\trans B\resp{v}\itk \\
0 \\
\frac{1}{\ggm\itk\tau\itk} \resp{\gl}\itk - \frac{\ggm\itk-1}{\ggm\itk}B\resp{v}\itk
\end{pmatrix}\!. \label{eq:kpvi}
\end{align}
}%


We then apply  VI (\ref{eq:optvi}), (\ref{eq:gvi}), and (\ref{eq:kpvi}) in order to prove the following lemmas for our contraction proof, which show that the difference
between iterates decreases as the iterates approach the true
solution. `The remaining details of the proof are in the supplementary material.

\begin{lemma}\label{lm1}
The iterates $z\itk = (u\itk, v\itk, \gl\itk)\trans$ generated by ADMM satisfy
\begin{align}
(B\resp{v}\itk)\trans \resp{\gl}\itk \geq 0.
\end{align}
\end{lemma}


\begin{lemma}\label{lm3}
Let $\ggm\itk \geq 1.$ The optimal solution $z\opt$ and iterates $z\itk$ generated by ADMM satisfy
\begin{equation}
\begin{split}
 \frac{2-\ggm\itk}{\ggm\itk}& \| \tau\itk B \resp{v}\itk + \resp{\gl}\itk \|^2 \\
\leq & \ggm\itk(\| \tau\itk B \reso{v}\itk \|^2 + \| \reso{\gl}\itk \|^2) \\
&- (2-\ggm\itk) (\| \tau\itk B \reso{v}\kp \|^2 + \|\reso{\gl}\kp\|^2).
\end{split}
\label{eq:lm3}
\end{equation}
\end{lemma}

\subsection{Convergence with adaptivity}
\label{sec:thm}

We are now ready to state our main convergence results.  The proof of Theorem \ref{thmas1} is shown here in full, and leverages \cref{lm3} to produce a contraction argument.     The proof of Theorem \ref{thmas2} is extremely similar, and is shown in the supplementary material.

\begin{theorem}\label{thmas1}
 Suppose \cref{as1} holds.
 Then, the iterates 
 $z\itk = (u\itk, v\itk, \gl\itk)\trans$
 generated by ADMM satisfy
\begin{align}
\lim_{k\rightarrow \infty} \| r\itk \| = 0 \quad \text{and} \quad \lim_{k\rightarrow \infty} \| d\itk \| = 0.
\end{align}
\end{theorem}

\begin{proof}
\cref{as1} implies
\begin{align}
\frac{\ggm\itk}{2-\ggm\itk}\tau\itk^2 \leq (1+\eta\itk^2) \tau\km^2 \ \text{and}
\ \frac{\ggm\itk}{2-\ggm\itk} \leq (1+\eta\itk^2). \label{eq:th1a}
\end{align}
If $\ggm\itk < 2$ as in \cref{as1}, then \cref{lm3} shows
\begin{align}
& \frac{1}{\ggm\itk} \| \tau\itk B \resp{v}\itk + \resp{\gl}\itk \|^2  \nonumber\\
\leq & \frac{\ggm\itk}{2-\ggm\itk}(\tau\itk^2 \|  B \reso{v}\itk \|^2 + \| \reso{\gl}\itk \|^2) \nonumber\\
& \quad-  (\tau\itk^2 \|  B \reso{v}\kp \|^2 + \|\reso{\gl}\kp\|^2) \label{eq:th1t1}\\
\leq & (1+\eta\itk^2)(\tau\km^2 \|  B \reso{v}\itk \|^2 + \| \reso{\gl}\itk \|^2) \nonumber\\
& \quad-  (\tau\itk^2 \|  B \reso{v}\kp \|^2 + \|\reso{\gl}\kp\|^2), \label{eq:th1t2}
\end{align}
where \eqref{eq:th1a} is used to get from \eqref{eq:th1t1} to \eqref{eq:th1t2}. Accumulating inequality \eqref{eq:th1t2} from $k=0$ to $N$ shows
\begin{align}\label{eq:th1t3}
 \sum_{k=0}^{N}& \prod_{t=k+1}^{N} (1+\eta_t^2) \frac{1}{\ggm\itk} \| \tau\itk B \resp{v}\itk + \resp{\gl}\itk \|^2 \nonumber\\
 &\leq  \prod_{k=1}^{N} (1+\eta_t^2) (\tau_0^2 \|  B \reso{v}_0 \|^2 + \| \reso{\gl}_0 \|^2).
\end{align}
\cref{as1} also implies $\prod_{t=1}^{\infty} (1+\eta_t^2) \!\! < \!\! \infty$, and $\prod_{t=k+1}^{N} (1+\eta_t^2) \frac{1}{\ggm\itk} \!\! \geq \!\! \frac{1}{\ggm\itk} \!\! >\!\! \half$. Then, \eqref{eq:th1t3} indicates $\sum_{k=0}^{\infty} \| \tau\itk B \resp{v}\itk + \resp{\gl}\itk \|^2 < \infty,$ and
\begin{align}
\lim_{k\rightarrow \infty} \| \tau\itk B \resp{v}\itk + \resp{\gl}\itk \|^2 = 0.
\end{align}
Now, from \cref{lm1}, $(B\resp{v}\itk)\trans \resp{\gl}\itk \geq 0,$ and so
{\small
\begin{align}
&\lim_{k\rightarrow \infty} \| \resp{\gl}\itk \|^2 \leq \lim_{k\rightarrow \infty} \| \tau\itk B \resp{v}\itk + \resp{\gl}\itk \|^2  = 0,\\
& \lim_{k\rightarrow \infty} \| \tau\itk   B\resp{v}\itk\|^2 \leq \lim_{k\rightarrow \infty} \| \tau\itk B \resp{v}\itk + \resp{\gl}\itk \|^2 = 0.
\end{align}
}%
The residuals $r\itk, d\itk$ in \eqref{eq:res} satisfy
 \begin{align}
 & r\itk = \frac{1}{\ggm\itk\tau\itk} \resp{\gl}\km -\frac{\ggm\itk-1}{\ggm\itk} B\resp{v}\km, \\
 &
  d\itk = \tau\itk A\trans B\resp{v}\km,
 \end{align}
 from which we get
\begin{multline}
 \lim_{k\rightarrow \infty} \| r\itk \| \leq  \lim_{k\rightarrow \infty} \frac{1}{\ggm\itk\tau\itk} \|\resp{\gl}\km\| \nonumber\\
 + \frac{\ggm\itk-1}{\ggm\itk\tau\km ^2} \|\tau\km B\resp{v}\km\| = 0, \text{ and }
\end{multline}
\begin{align*}
& \lim_{k\rightarrow \infty} \| d\itk \| \leq  \lim_{k\rightarrow \infty} \|A\| \| \tau\itk B\resp{v}\km \| \\
& \qquad  \leq  \lim_{k\rightarrow \infty} \sqrt{1+\eta\itk^2}\|A\| \, \| \tau\km B\resp{v}\km \|  = 0.
\end{align*}
\phantom{easdasd}\\[-1.0cm]
\end{proof}

Similar methods can be used to prove the following about convergence under  \cref{as2}.  The proof of the following theorem is given in the supplementary material.
\begin{theorem}\label{thmas2}
Suppose \cref{as2} holds. Then, the iterates  
$z\itk = (u\itk, v\itk, \gl\itk)\trans$
generated by ADMM satisfy
\begin{align}
\lim_{k\rightarrow \infty} \| r\itk \| = 0 \quad \text{and} \quad \lim_{k\rightarrow \infty} \| d\itk \| = 0.
\end{align}
\end{theorem}

\section{Applications}
\label{sec:app}
We focus on the following statistical and image processing problems involving non-differentiable objectives:
linear regression with elastic net regularization (EN), low-rank least squares (LRLS), quadratic programming (QP), consensus $\ell_1$-regularized logistic regression, support vector machine (SVM), total variation image restoration (TVIR), and robust principle component analysis (RPCA). We study several vision benchmark datasets such as the extended Yale~B face dataset \cite{georghiades2001few}, MNIST digital images \cite{lecun1998gradient}, and CIFAR10 object images\footnote{We use the first batch of CIFAR10 that contains $10000$ samples.} \cite{krizhevsky2009learning}. We also use synthetic and benchmark datasets from ~\cite{efron2004least,zou2005regularization,lee2006efficient,schmidt2007fast,liu2009large,goldstein2014fast}, which are obtained from the UCI repository and the LIBSVM page. The experimental setups for each problem are briefly described here, and the implementation details are provided in the supplementary material.

\vspace{-0.4cm}
\paragraph{Linear regression with EN regularization} Elastic net (EN) is a modification of the $\ell_1$-norm (or LASSO) regularizer that helps dealing with highly correlated variables~\cite{zou2005regularization,goldstein2014fast}, and requires solving
\begin{eqnarray}
\min_x \frac{1}{2} \| Dx - c \|_2^2 + \rho_1 \|x\|_1 + \frac{\rho_2}{2} \| x\|_2^2, \label{eq:enet}
\end{eqnarray}
where $\|\cdot\|_1 $ denotes the $\ell_1$-norm, $D$ is the data matrix, $c$ contains measurements, and $x$ is the vector of regression coefficient. 

\vspace{-0.4cm}
\paragraph{Low-rank least squares (LRLS)}
The nuclear norm (the $\ell_1$-norm of the matrix singular values) is a convex surrogate for matrix rank. ADMM has been applied to solve low rank least squares problems~\cite{yang2013linearized,xu2015bmvc}
\begin{eqnarray}
\min_{X} \frac{1}{2} \|DX - C\|_F^2 + \rho_1 \| X \|_* + \frac{\rho_2}{2} \| X\|_F^2,
\end{eqnarray}
where $\|\cdot\|_* $ denotes the nuclear norm, $\|\cdot\|_F $ denotes the Frobenius norm, $D\in \bbR^{n \times m}$ is a data matrix, $C\in \bbR^{n \times d}$ contains measurements, and $X\in \bbR^{m \times d}$ contains variables. 

ADMM is applied by splitting the regression term and the non-differentiable regularizer composed of nuclear and Frobenius norm. LRLS has been used to formulate exemplar classifiers and discover visual subcategories \cite{xu2015bmvc}.

\vspace{-0.4cm}
\paragraph{ SVM and QP} Support vector machine (SVM) is one of the most successful binary classifiers for computer vision. The dual  of the SVM is a QP problem,
\begin{align*}
\min_{z}  &\quad \frac{1}{2} z^{T}Qz - e^{T}z  \\
\st & \quad  c^{T}z = 0 ~\mbox{and}~ 0 \leq z \leq C,
\end{align*}
where $z$ is the SVM dual variable, $Q$ is the kernel matrix, $c$ is a vector of labels,  $e$ is a vector of ones, and  $C > 0$~\cite{chang2011libsvm}. The canonical QP is also considered,
\begin{equation}
\min_x \frac{1}{2}x^{T}Qx + q^{T}x~~~~\st~~Dx \leq c. \label{eq:qp}
\end{equation}

\begin{table*}[t]
\centering
\caption{\small Iterations (and runtime in seconds) for various applications. Absence of convergence after $n$ iterations is indicated as $n+$.
}
\setlength{\tabcolsep}{3pt}
\small
\begin{threeparttable}
\begin{tabular}{|c|c|c||c|c|c|c|>{\bfseries}c|}
\hline
Application & Dataset & \tabincell{c}{\#samples $\times$ \\ \#features\tnote{1}
} & \tabincell{c}{Vanilla\\ ADMM} & \tabincell{c}{Relaxed \\ADMM} & \tabincell{c}{Residual\\ balance} & \tabincell{c}{Adaptive \\ADMM} &  \tabincell{c}{Proposed \\ ARADMM} \\
\hline\hline
\multirow{6}{*}{\tabincell{c}{Elastic net\\regression}}
& Synthetic & 50 $\times$ 40 & 2000+(.642) & 2000+(.660) & 424(.144) & 102(.051) & 70(.026) \\
& MNIST & 60000 $\times$ 784 & 1225(29.4) & 816(19.9) & 94(2.28) & 41(.943) & 21(.549) \\
& CIFAR10 &  10000 $\times$ 3072  & 2000+(690) & 2000+(697) & 556(193) & 2000+(669) & 94(31.7) \\
& News20 & 19996 $\times$ 1355191 & 2000+(1.21e4) & 2000+(9.16e3) & 227(914) & 104(391) & 71(287) \\
& Rcv1 & 20242 $\times$ 47236  & 2000+(1.20e3) & 1823(802) & 196(79.1) & 104(35.7) & 64(26.0) \\
& Realsim & 72309 $\times$ 20958 & 2000+(4.26e3) & 2000+(4.33e3) & 341(355) & 152(125) & 107(88.2) \\
\hline
\multirow{5}{*}{\tabincell{c}{Low rank \\least squares}}
&Synthetic & 1000 $\times$ 200 & 2000+(118) & 2000+(116) & 268(15.1) & 26(1.55) & 18(1.04) \\
&German & 1000 $\times$ 24 & 2000+(4.72) & 2000+(4.72) & 642(1.52) & 130(.334) & 52(.125) \\
&Spectf & 80 $\times$ 44 & 2000+(2.70) & 2000+(2.74) & 336(.455) & 162(.236) & 105(.150) \\
& MNIST & 60000 $\times$ 784 & 200+(1.86e3) & 200+(2.08e3) & 200+(3.29e3) & 200+(3.46e3) & 38(658) \\
& CIFAR10 & 10000 $\times$ 3072 & 200+(7.24e3) & 200+(1.33e4) & 53(1.60e3) & 8(208) & 6(156) \\
\hline
\multirow{3}{*}{\tabincell{c}{QP and \\dual SVM}} &Synthetic & 250 $\times$ 500 & 1224(11.5) & 823(7.49) & 626(5.93) & 170(1.57) & 100(.914) \\
&German & 1000 $\times$ 24 & 2000+(58.8) & 2000+(61.8) & 1592(45.0) & 1393(38.9) & 1238(34.9) \\
& Spectf & 80 $\times$ 44 & 2000+(.846) & 2000+(.777) & 169(.070) & 175(.086) & 53(.026) \\
\hline
\multirow{5}{*}{\tabincell{c}{Consensus \\ logistic \\regression}}
& Synthetic & 1000 $\times$ 25 & 590(9.93) & 391(6.97) & 70(1.23) & 35(.609) & 20(.355) \\
&German & 1000 $\times$ 24 & 2000+(34.3) & 2000+(66.6) & 151(2.60) & 35(.691) & 26(.580) \\
&Spectf & 80 $\times$ 44 & 1005(20.1) & 667(14.4) & 117(1.98) & 145(1.63) & 85(1.07) \\
&MNIST & 60000 $\times$ 784 & 200+(2.99e3) & 200+(3.47e3) & 200+(1.37e3) & 49(536) & 28(333) \\
&CIFAR10 & 10000 $\times$ 3072 & 200+(593) & 200+(2.08e3) & 200+(1.54e3) & 131(165) & 19(33.7) \\
\hline
\multirow{5}{*}{\tabincell{c}{Unwrapping \\ SVM}}
&Synthetic & 1000 $\times$ 25 & 2000+(1.13) & 1418(.844) & 2000+(1.16) & 355(.229) & 147(.094) \\
&German & 1000 $\times$ 24 & 753(1.88) & 560(1.37) & 2000+(4.98) & 572(1.44) & 213(.545) \\
&Spectf & 80 $\times$ 44 & 567(.203) & 367(.112) & 567(.185) & 207(.068) & 149(.052) \\
& MNIST & 60000 $\times$ 784 & 128(130) & 118(111) & 163(153) & 200+(217) & 67(71.0) \\
& CIFAR10 & 10000 $\times$ 3072 & 200+(512) & 200+(532) & 200+(516) & 89(285) & 57(143) \\
\hline
\multirow{3}{*}{\tabincell{c}{Image \\ denoising}}
&Barbara & 512 $\times$ 512 & 262(35.0) & 175(23.6) & 74(10.0) & 59(8.67) & 38(5.57) \\
&Cameraman & 256 $\times$ 256 & 311(8.96) & 208(5.89) & 82(2.29) & 88(2.76) & 35(1.08) \\
&Lena & 512 $\times$ 512 & 347(46.3) & 232(31.3) & 94(12.5) & 68(9.70) & 39(5.58) \\
\hline
\multirow{3}{*}{\tabincell{c}{Robust \\ PCA}}
& FaceSet1 & 64 $\times$ 1024 & 2000+(41.1) & 1507(30.3) & 560(11.1) & 561(11.9) & 267(5.65) \\
& FaceSet2 & 64 $\times$ 1024 & 2000+(41.1) & 2000+(41.4) & 263(5.54) & 388(9.00) & 188(4.02) \\
& FaceSet3 & 64 $\times$ 1024 & 2000+(39.4) & 1843(36.3) & 375(7.44) & 473(9.89) & 299(6.27) \\
\hline
\end{tabular}%
\label{tab:exp}%
\begin{tablenotes}
    \item[1] \#constrains $\times$ \#unknowns for canonical QP; width $\times$ height for image restoration. 
     \vspace{-2mm}
\end{tablenotes}
\end{threeparttable}
\end{table*}%

\vspace{-0.5cm}
\paragraph{Consensus $\ell_1$-regularized logistic regression}
ADMM has become an important tool for solving distributed optimization problems~\cite{boyd2011admm}. A typical problem is the consensus $\ell_1$-regularized logistic regression
\begin{equation}
\begin{split}
&\min_{x_i, z} \sum_{i=1}^{N} \sum_{j=1}^{n_i} \log(1+\exp(-c_jD_jx_i)) + \rho \|z\|_1 \\
&\st~~ x_i - z = 0, i=1,\ldots,N,
\end{split}
\end{equation}
where $x_i \in \bbR^m$ represents the local variable on the $i$th distributed node, $z$ is the global variable, $n_i$ is the number of samples in the $i$th block, $D_j \in \bbR^m$ is the $j$th sample, and $c_j\in \{-1,+1\}$ is the corresponding label.

\vspace{-0.4cm}
\paragraph{Unwrapped SVM}
The unwrapped formulation of SVM~\cite{goldstein2016unwrapping}, which can be used in distributed computing environments via ``transpose reduction'' tricks, applies ADMM to the primal form of SVM to solve
\begin{equation}
\min_x \frac{1}{2} \| x\|_2^2 +  C \sum_{j=1}^{n} \max\{1-c_jD_j^Tx, \, 0\},
\end{equation}
where $D_j \in \bbR^m$ is the $j$th sample of training data, and $c_j\in \{-1,1\}$ is the corresponding label. ADMM is applied by splitting the $\ell_2$-norm regularizer and the non-differentiable hinge loss term.

\vspace{-0.4cm}
\paragraph{Total variation image denoising (TVID)}
Total variation image denoising is often performed by solving \cite{rudin1992nonlinear}
\begin{eqnarray}
\min_{x} \frac{1}{2} \| x-c \|_2^2 + \rho \| \grad x\|_1
\end{eqnarray}
where $c$ represents given noisy image, and $\grad$ is the discrete gradient operator, which computes differences between adjacent image pixels. 
ADMM is applied by splitting the $\ell_2$-norm term and the non-differentiable total variation term.

\vspace{-0.4cm}
\paragraph{RPCA}
Robust principal component analysis (RPCA) has broad applications in computer vision and imaging \cite{wright2009robust,naikal2011informative,ozyesil2015robust}. RPCA recovers a low-rank matrix and a sparse matrix by solving
\begin{eqnarray}
\min_{Z, E} \| Z \|_* + \rho \| E \|_1 ~~\st~~ Z + E =C,
\end{eqnarray}
where the nuclear norm $\| \cdot \|_*$ is used to obtain a low rank matrix $Z$, and $\| \cdot \|_1$ is used to obtain a sparse error $E$.

\begin{figure*}[tbhp]
\centerline{
\includegraphics[width=0.95\linewidth]{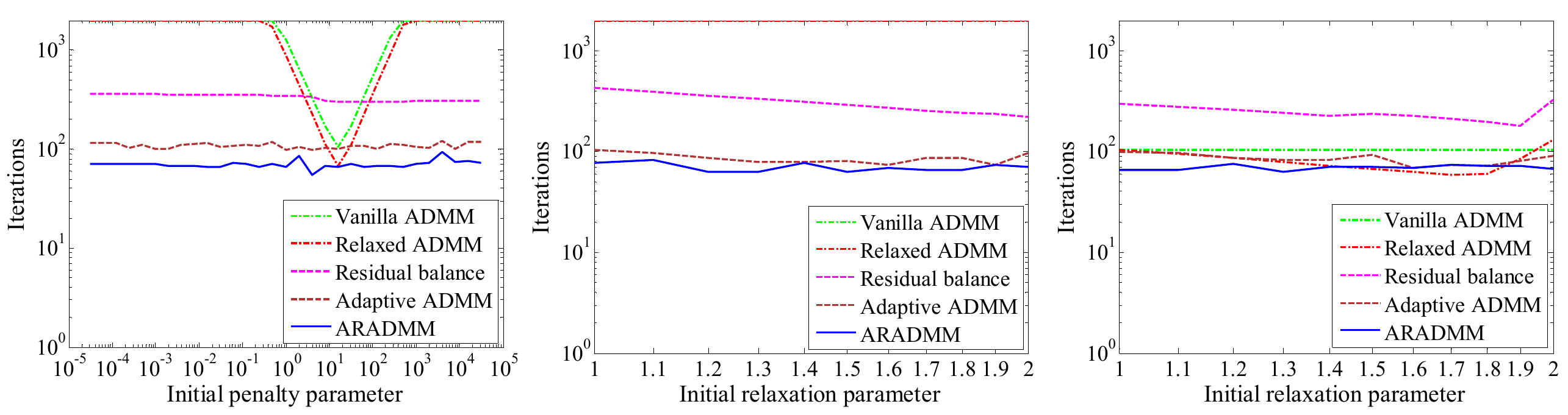}
}
\caption{\small Sensitivity of convergence speed for the synthetic problem of EN regularized linear regression. (left) sensitivity to the initial penalty $\tau_0$; (middle) sensitivity to relaxation $\ggm_0$; (right) sensitivity to relaxation $\ggm_0$ when optimal $\tau_0$ is selected by grid search.}
\label{fig:en}
\vspace{-2mm}
\end{figure*}

\begin{figure}[tbhp]
\vspace{-3mm}
\centerline{
\includegraphics[width=0.7\linewidth]{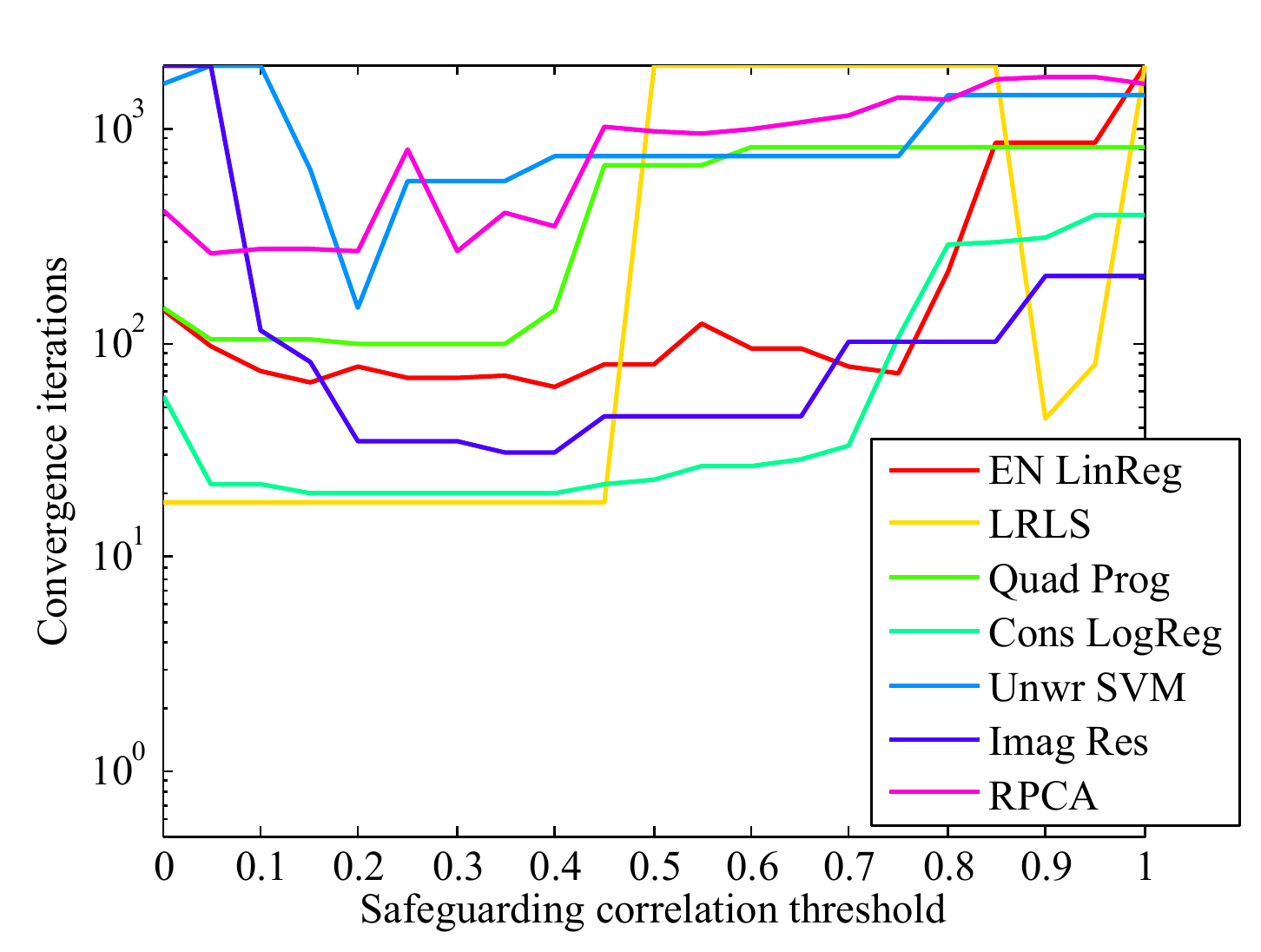}
}
\caption{\small Sensitivity of convergence speed to safeguarding threshold $\gep^{{\scriptsize \text{cor}}}$ for proposed ARADMM. Synthetic problems ('cameraman' for TVIR, and 'FaceSet1' for RPCA) of various applications are studied. Best viewed in color.
 }  
\label{fig:corr}
\vspace{1mm}
\end{figure}

\section{Experiments} \label{sec:exp}

The proposed AADMM is implemented as shown in Algorithm~\ref{alg}.
 We also implemented vanilla ADMM, (non-adaptive) relaxed ADMM, ADMM with residual balancing (RB), and adaptive ADMM (AADMM) for comparison.

 The relaxation parameter for the non-adaptive relaxed ADMM is fixed at $\ggm\itk \! = \! 1.5$ as suggested in \cite{eckstein1992douglas}. The parameters of RB and AADMM are selected as in \cite{he2000alternating,boyd2011admm,xu2016adaptive}.
The initial penalty $\tau_0 \! = \! \nicefrac{1}{10}$ and initial relaxation $\ggm_0 \! =\! 1$ are used for all problems except the canonical QP problem, where initial parameters are set to the geometric mean of the maximum and minimum eigenvalues of matrix $Q$, as proposed for quadratic problems in~\cite{raghunathan2014alternating}. 

For each problem, the same randomly generated  initial variables  $v_0, \gl_0$ are used for ADMM and its variant methods. As suggested by \cite{he2000alternating,xu2016adaptive}, the adaptivity of RB and AADMM is stopped after 1000 iterations to guarantee convergence.

\subsection{Convergence results}
Table~\ref{tab:exp} reports the convergence speed of ADMM and its variants for the applications described in \cref{sec:app}. More experimental results including the table of more test cases, the convergence curves, and visual results of image restoration and robust PCA for face decomposition are provided in the supplementary material.
Relaxed ADMM often outperforms vanilla ADMM, but does not compete with adaptive methods like RB, AADMM and ARADMM. The proposed ARADMM performs best in all the test cases.

\subsection{Sensitivity to initialization} \label{sec:sensitivity}
We study the sensitivity of the different ADMM variants to the initial penalty ($\tau_0$) and initial relaxation parameter ($\ggm_0$). Fig.~\ref{fig:en} presents iteration counts for a wide range of values of $\tau_0, \ggm_0$, for elastic net regression with synthetic datasets. In the left and center plots we fix one of $\tau_0, \ggm_0$ and vary the other.  The number of iterations needed to convergence is plotted as the algorithm parameters vary. In the right plot, we use a grid search to find the {\em optimal} $\tau_0$ for different values of $\ggm_0$.  Fig.~\ref{fig:en} (left) shows that adaptive methods are relatively stable with respect to the initial penalty $\tau_0$, while ARADMM outperforms RB and AADMM in all choices of initial $\tau_0$.  Fig.~\ref{fig:en} (middle) suggests that the relaxation $\ggm_0$ is generally less important than $\tau_0$. When a bad value of $\tau$ is chosen, it is unlikely that a good choice of~$\ggm$ can compensate. The proposed ARADMM that jointly adjusts $\tau,\ggm$ is generally better than simply adding the relaxation to the existing adaptive methods RB and AADMM.

Fig.~\ref{fig:en} (right) shows the sensitivity to $\ggm$ when using a grid search to choose the optimal $\tau_0$.  This optimal $\tau_0$ significantly improves the performance of vanilla ADMM and relaxed ADMM (which use the same $\tau_0$ for all iterations). Even when using the optimal stepsize for the non-adaptive methods, ARADMM is superior to or competitive with the non-adaptive methods.  Note that this experiment is meant to show a best-case scenario for the non-adaptive methods;  in practice the user generally has no knowledge of the optimal value of $\tau.$ Adaptive methods achieve optimal or near-optimal performance without an expensive grid search.

\subsection{Sensitivity to safeguarding} 
Finally, \cref{fig:corr} presents iteration counts when applying ARADMM with various safeguarding correlation thresholds $\gep^{{\scriptsize \text{cor}}}$.  When $\gep^{{\scriptsize \text{cor}}} =0$, the calculated adaptive parameters based on curvature estimations are always accepted, and when $\gep^{{\scriptsize \text{cor}}} \! =\! 1$ the parameters are never changed.  The proposed AADMM method is insensitive to $\gep^{{\scriptsize \text{cor}}}$ and performs well for a wide range of $\gep^{{\scriptsize \text{cor}}} \in [0.1, \, 0.4]$ for various applications, except for unwrapping SVM and RPCA. Though tuning such ``hyper-parameters'' may improve the performance of ARADMM for some applications, the fixed $\gep^{{\scriptsize \text{cor}}} = 0.2$ performs well in all our experiments (seven applications and over fifty test cases, a full list is in the supplementary material). The proposed ARADMM is fully automated and performs well without parameter tuning.


\section{Conclusion}
We have proposed an adaptive method for jointly  tuning the penalty and relaxation parameters of relaxed ADMM without user oversight. We have analyzed adaptive relaxed ADMM schemes, and provided conditions for which convergence is guaranteed.  Experiments on a wide range of machine learning, computer vision, and image processing benchmarks have demonstrated that the proposed adaptive method (often significantly) outperforms other ADMM variants without 
user oversight or 
parameter tuning.  The new adaptive method improves the applicability of relaxed ADMM by facilitating fully automated solvers that exhibit fast convergence and are usable by non-expert users. 


\subsubsection*{Acknowledgments}
TG and ZX were supported by the US Office of Naval Research under grant N00014-17-1-2078 and by the US National Science Foundation (NSF) under grant CCF-1535902. MF was partially supported by the Funda\c{c}\~{a}o para a
Ci\^{e}ncia e Tecnologia, grant UID/EEA/5008/2013. XY was supported by the General Research Fund from Hong Kong Research Grants Council under grant HKBU-12313516. CS was supported in part by Xilinx Inc., and by the US NSF under grants ECCS-1408006, CCF-1535897, and CAREER CCF-1652065.

\balance
{\small
\bibliographystyle{ieee}
\bibliography{admm,tensor,zheng}
}
\balance

\end{document}